\documentclass{article}

\usepackage[preprint]{neurips_2021}
\usepackage{subcaption}

\usepackage[utf8]{inputenc} 
\usepackage[T1]{fontenc}    
\usepackage{hyperref}       
\usepackage{url}            
\usepackage{booktabs}       
\usepackage{amsfonts}       
\usepackage{nicefrac}       
\usepackage{microtype}      
\usepackage{xcolor}         

\title{Be Considerate: Objectives, Side Effects, and Deciding How to Act}

\author{Parand Alizadeh Alamdari, Toryn Q. Klassen, Rodrigo Toro Icarte,
Sheila A. McIlraith$^\dagger$ \\
Department of Computer Science, University of Toronto, Toronto, Canada\\
Vector Institute, Toronto, Canada\\
$^\dagger$~Schwartz Reisman Institute for Technology and Society, Toronto, Canada\\
\{parand,toryn,rntoro,sheila\}@cs.toronto.edu}

\usepackage{pdfpages}
\usepackage{wrapfig}
\usepackage{enumitem}  
\newcommand{\hideit}[1]{}

\usepackage{mathtools}
\usepackage{amssymb}
\usepackage{amsmath}
\usepackage{xspace}
\usepackage{color}
\usepackage{dsfont}

\newcommand{\caring}{\text{caring}}
\newcommand{\Caring}{\text{Caring}}
\newcommand{\craft}{\text{Craft}}

\newcommand{\expect}{\mathbb{E}}

\newcommand{\tuple}[1]{\langle{#1}\rangle}

\usepackage{amsthm}
\newtheorem{prop}{Proposition}
\theoremstyle{definition}
\newtheorem{definition}{Definition}

\newcommand{\initset}{\mathcal{I}}
\newcommand{\Pnext}{\mathsf{P}}
\newcommand{\rnextv}{r_\mathsf{value}}
\newcommand{\rnextnegv}{r_\mathsf{neg}}
\newcommand{\rnexto}{r_\mathsf{option}}
\newcommand{\rnextvp}{r_\mathsf{value}'}
\newcommand{\rnextminv}{r_\mathsf{min}}
\newcommand{\rnextworstv}{r_\mathsf{worst}}
\newcommand{\rnexts}{r_\mathsf{sim}}

\begin{document}


\maketitle

\begin{abstract}
Recent work in AI safety has highlighted that in sequential decision making, objectives are often underspecified or incomplete. This gives discretion to the acting agent to realize the stated objective in ways that may result in undesirable outcomes. We contend that to learn to act safely, a reinforcement learning (RL) agent should include contemplation of the impact of its actions on the wellbeing and agency of others in the environment, including other acting agents and reactive processes. We endow RL agents with the ability to contemplate such impact by augmenting their reward based on expectation of future return by others in the environment, providing different criteria for characterizing impact. We further endow these agents with the ability to differentially factor this impact into their decision making, manifesting behavior that ranges from self-centred to self-less, as demonstrated by experiments in gridworld environments.
\end{abstract}

\section{Introduction}\label{sec:introduction}

Robust and reliable sequential decision making, 
whether it is realized via reinforcement learning (RL), supervised learning, or some form of probabilistic or otherwise symbolic planning using models---relies on the specification of an objective -- a reward function to be optimized in the case of RL, or a goal to achieve in the case of symbolic planning. 
Recent work in AI safety has raised the concern that objective specifications are often underspecified or incomplete, neglecting to take into account potential undesirable (negative) side effects of achievement of the specified objective. As \cite{Amodei2016concrete} explain, 
``[F]or an agent operating in a large, multifaceted environment, an objective function that focuses on only one aspect of the environment may implicitly express indifference over other aspects of the environment.''  \cite{Russell2019compatible} cites the example of tasking a robot to get coffee from a coffee shop and the robot, in its singular commitment to achieving the stated objective, killing all those in the coffee shop that stood between it and the purchase of coffee. 
A more benign example is that of a robot breaking a vase that is on the optimal path between two points \citep{Krakovna2020avoiding}. 
A range of recent works have presented computational techniques for avoiding or learning to avoid negative side effects in Markov Decision Processes (MDPs) or related models, including for example \citep{Zhang2018minimax,Krakovna2019stepwise,Turner2020conservative,Krakovna2020avoiding,Saisubramanian2020multi}.

Our concern in this paper is with how a Reinforcement Learning (RL) agent can learn to act safely in the face of a potentially incomplete specification of the objective.  Is avoiding negative side effects the answer?  
\citeauthor{Amodei2016concrete} observe that 
\begin{quote}
``avoiding side effects can be seen as a proxy for the things we really care about:  avoiding negative externalities. If everyone likes a side effect, there's no need to avoid it." 
\end{quote}
In the spirit of this observation,   
we contend that to act safely an agent should contemplate the impact of its actions on the wellbeing and agency of others in the environment, including both other acting agents as well as passive or reactive processes such as climate or waterways. 
Indeed, what may be construed as a positive effect for one agent may be a negative effect for another. 

Here we endow RL agents with the ability to 
consider in their learning
the future welfare and continued agency of those that act or react in the environment. This is in contrast to some recent work on side effects that takes into account only how the agent's actions will affect its own future abilities  \cite{Krakovna2019stepwise,Turner2020conservative,Krakovna2020avoiding}. We do so by augmenting the RL agent's reward with an auxilliary reward that reflects different functions of expected future return. In its most general case, we make no assumptions about the number of agents that exist in the environment, their actions, or their transition functions. However we show how individual or groups of agents can be distinguished and differentially considered in learning. We further endow these agents with the ability to control the degree to which impact on self versus others factors into their learning, resulting in behavior that ranges from self-centred to self-less.  Experiments in gridworld environments illustrate qualitative and quantitative properties of the proposed approach.

\section{Preliminaries}
\label{sec:preliminaries}

In this section, we review relevant definitions and notation.

RL agents learn policies from experience. When the problem is fully-observable, it's standard to model the environment as 
a Markov Decision Process (MDP) \citep{Sutton2018textbook}.
We describe an MDP as a tuple $\tuple{S,A,T, r,\gamma}$ where  $S$ is a finite set of states, $A$ is a finite set of actions,  $T(s_{t+1}|s_t,a_t)$ gives the probability of transitioning to state $s_{t+1}$ when taking action $a_t$ in state $s_t$, $r:S\times A\times S\to \mathbb{R}$ is the reward function, and  $\gamma$ is the discount factor.
Sometimes an MDP can also include a designated initial state $s_0\in S$.

When the agent takes an action $a_t\in A$ in a state $s_t\in S$, as result it ends up in a new state $s_{t+1}$ drawn from the distribution $T(s_{t+1}|s_t,a_t)$, and receives the reward $r_{t+1}=r(s_t,a_t,s_{t+1})$.
A \emph{terminal state} in an MDP is a state $s$ which can never be exited -- i.e., $T(s|s,a)=1$ for every action $a$ -- and from which no further reward can be gained -- i.e., $r(s,a,s)=0$ for every action $a$. 

A \emph{policy} is a (possibly stochastic) mapping from states to actions.
Given a policy $\pi$, the \emph{value} of a state $s$ is the \emph{expected return} of that state, that is, the expected sum of (discounted) rewards that the agent will get by following the policy $\pi$ starting in $s$. That can be expressed using the value function $V^\pi$, defined as
$V^\pi(s)=\expect_\pi\left[\sum_{k=0}^\infty \gamma^k \cdot r_{t+k+1} \mid s_t=s\right]$
where the $\expect_\pi$ notation means that in the expectation, the action in each state $s_t,s_{t+1},\dots$ is selected according to the policy $\pi$. 
Note that the discount factor $\gamma$ determines how much less valuable it is to receive rewards in the future instead of the present. An optimal policy will maximize the value of every state. The optimal value function $V^*$ is the value function of an optimal policy.

Similarly to with the value function, a value (called a Q-value) can be associated with a state-action pair: $
    Q^\pi(s,a)=\expect_\pi\left[\sum_{k=0}^\infty \gamma^k\cdot r_{t+k+1} \mid s_t=s,a_t=a\right].$
Note that here, the first action selected is necessarily $a$, but the policy $\pi$ is followed afterwards. 
The optimal Q-function is the Q-function $Q^*$ corresponding to an optimal policy.
Note also that given the optimal Q-function $Q^*$, an optimal policy can be recovered as follows: for each state $s$, choose an action $a$ such that $Q^*(s,a)$ is maximized.

In RL, the agent is tasked with trying to maximize reward, but is not given complete knowledge about the environment's dynamics (i.e., the transition function $T$ and reward function $R$). As a result, the agent has to explore to learn how to get reward. We refer the reader to \citeauthor{Sutton2018textbook}'s \citeyear{Sutton2018textbook} textbook for further details on RL, but note that Q-learning \citep{Watkins1992qlearning} is an RL algorithm that learns a policy by learning approximations of the optimal Q-values (without ever learning the transition dynamics).

{\hideit{
In partially observable problems, the underlying environment model is typically assumed to be a \emph{Partially Observable Markov Decision Process (POMDP)}. A POMDP is a tuple $\mathcal{P_O} = \tuple{S,O,A,r,p,\omega,\gamma}$, where $S$, $A$, $r$, $p$, and $\gamma$ are defined as in an MDP, $O$ is a finite set of \emph{observations}, and $\omega(s,o)$ is the \emph{observation probability distribution}. 
At every time step $t$, the agent is in exactly one state $s_t \in S$, executes an action $a_t \in A$, receives reward $r_{t} = r(s_t,a_t)$, and moves to state $s_{t+1}$ according to $p(s_t,a_t,s_{t+1})$. However, the agent does not observe $s_{t+1}$, but only receives an observation $o_{t+1} \in O$. This observation provides the agent a clue about what the state $s_{t+1} \in S$ is via $\omega$. In particular, $\omega(s_{t+1},o_{t+1})$ is the probability of observing $o_{t+1}$ from state $s_{t+1}$ \cite{cassandra1994acting}.
}}

{}


%
%
%


\section{Approach: Considering others}
\label{sec:approach}
%

In Section \ref{sec:introduction}, we suggested that to act safely an acting agent should contemplate the impact of its actions on the welfare and continued agency of those that act or react in the environment. In this section, we present several formulations that address this aspiration. 
To this end, we conceive an environment which may 
feature various, potentially un-named, agents. We consider a simplified setting in which we assume that first the acting agent acts in isolation, but after its policy leads to a terminal state, other agents can act. Those others are not learning their behavior; rather, they subscribe to some previously learned or acquired policies or behaviors. 
We assume the acting agent must learn how to act in consideration of others, \emph{but} with limited knowledge of others in the environment. In the most general case, the acting agent need not know what agents exist in the environment, nor their objectives, actions or transition systems. Rather we assume a distribution over value functions that serves to capture how the world may purposefully evolve in the future, and we use this as a proxy for evaluating the welfare and continued agency of those that act or react in the environment. 

In what follows, we present several different formulations for augmenting the reward accrued by the acting agent, by considering its effect on others in the environment. In all cases, this accrued reward is some function of the distribution over value functions, and we present several formulations 
that optimize for different criteria. These formulations are consistently paired with  ``\caring'' coefficients, hyperparameters that control the degree to which the acting agent factors into its decision making its effect on others versus itself. This enables the agent's decision making to range from selfish -- oblivious to its impact on others -- all the way to selfless, where it subjugates its own reward in service of the welfare of others.

In Section \ref{sec:experiments} we illustrate qualitative behavior manifest by different argumented rewards, and the quantitative impact of the caring coefficients on agents' reward.


\subsection{Using information about value functions}
We wish to augment our acting (RL) agent's reward with an auxiliary reward that reflects the impact of its choice of actions on future agency and wellbeing of others in the environment. We assume the acting agent has access to a distribution over value functions. In particular,
suppose that we have a finite set $\mathcal{V}$ of possible value functions $V:S\to\mathbb{R}$, and a probability distribution 
$\Pnext(V)$ over that set.
Note that we don't have to commit to how many agents there are (or what their actions are). It could be that each $V\in\mathcal{V}$ corresponds to a different agent, that the set reflects all possible value functions of a unique agent, or anything in between. Also, each $V\in\mathcal{V}$ could reflect some aggregation of the value functions of all or some of the agents.
 
We can then modify the acting agent's reward function $r_1$ to incorporate an auxiliary reward, yielding the new reward function $\rnextv$ below:
\begin{align}\label{eq:rvalue}
    \rnextv(s,a,s')=\begin{cases}
    \alpha_1\cdot r_1(s,a,s')&\text{if $s'$ is not terminal}\\
    \alpha_1\cdot r_1(s,a,s')+ \gamma\cdot \alpha_2\cdot\displaystyle\sum_{V\in\mathcal{V}} \Pnext(V)\cdot V(s')&\text{if $s'$ is terminal}
    \end{cases}
\end{align}
Note that using this reward with an RL agent requires being able to identify a terminal state $s'$ when the agent's in it.

\paragraph{\Caring~coefficients:}The hyperparameters $\alpha_1$ and $\alpha_2$ are real numbers that determine the degree to which the aggregate reward, $\rnextv(s,a,s')$, is informed by the reward of the acting agent (via $\alpha_1$) and the discounted future return of the environment (via $\alpha_2$). 
As we will see in Section \ref{sec:experiments}, if $\alpha_1=1$ and $\alpha_2 =0$, we just get the original reward function where the acting agent only values its own reward. The agent is oblivious to its impact on others in its environment. If $\alpha_1=0$, then the acting agent entirely ignores any reward it garners directly from its actions.

Note that future activity does not have to start in exactly the same state at which the acting agent ended. $V$ can be defined so that $V(s')$ gives the expected return of future activity considered over a known distribution of starting states, given that the acting agent ended in $s'$.

Observe that the auxiliary reward, $\displaystyle\sum_{V\in\mathcal{V}} \Pnext(V)\cdot V(s')$, can be written as a function of the form $F(\mathcal{V},\Pnext,s')$, and can be thought of as assessing how good it is to terminate in $s'$.  In formulations that follow, we will simply modify $F(\mathcal{V},\Pnext,s')$ to manifest different reward.

Observe that for some possible reward functions for the acting agent and future value functions, the acting agent may have an incentive to avoid terminating states. The acting agent can avoid or delay the penalty for negative future return by not terminating.
This incentive would typically be undesirable. However, it can be shown that under some circumstances, the acting agent's optimal policy will be terminating.
\begin{definition}[terminating policy]
Given an MDP with terminal states, a policy $\pi$ is \emph{terminating} if the probability of eventually reaching a terminal state by following $\pi$ from any state is 1.
\end{definition}
The following proposition and proof are similar to \cite[Theorem 1]{Illanes2020symbolic}.

\begin{prop}Let $M=\tuple{S,A,T,r_1,\gamma}$ be an MDP where $\gamma=1$, the reward function $r_1$ is negative everywhere, and there exists a terminating policy. Suppose $\rnextv$ is the reward function constructed from $r_1$ according to \autoref{eq:rvalue}, using some distribution $\Pnext(V)$. Then any optimal policy for the MDP $M'=\tuple{S,A,T,\rnextv,\gamma}$ with the modified reward will terminate with probability 1.
\end{prop}
\begin{proof}
Suppose for contradiction that there is an optimal policy $\pi^*$ for $M'$ that is non-terminating. Then there is some state $s\in S$ so that the probability of reaching a terminal state from $s$ by following $\pi^*$ is some value $c<1$. Since rewards are negative everywhere, that means that $V^{\pi^*}(s)=-\infty$. On the other hand, any terminating policy gives a finite value to each state. Since there is a terminating policy for $M$ there is one for $M'$, and so $\pi^*$ cannot be an optimal policy.
\end{proof}

\subsubsection{Considering the worst case}

Rewarding the acting agent based on the expected future return isn't the only option. An alternative is to give the acting agent the ``worst case'' future return, i.e., the return corresponding to the value function (with non-zero probability) that gives the least value to $s'$, i.e., $F(\mathcal{V},\Pnext,s')=\min_{V\in\mathcal{V}:\Pnext(V)>0} V(s')$, and so
\begin{align}\label{eq:secondminreturn}
    \rnextworstv(s,a,s')=\begin{cases}
    \alpha_1\cdot r_1(s,a,s')&\text{if $s'$ is not terminal}\\
    \alpha_1\cdot r_1(s,a,s')+ \gamma\cdot \alpha_2\cdot\left(\displaystyle\min_{V\in\mathcal{V}:\Pnext(V)>0} V(s')\right)&\text{if $s'$ is terminal}
    \end{cases}
\end{align}
Using $\rnextworstv$ encourages the acting agent to maximize the worst case future return. This may be desirable for avoiding rare disastrous outcomes.

\subsubsection{Avoiding negative change}

The augmented reward functions we've considered so far don't just encourage the acting agent to avoid harming others, but they can encourage the acting agent to actively help others. We contrast this with the aspiration of avoiding negative side effects as 
explored by \cite{Krakovna2019stepwise,Krakovna2020avoiding} and others. In their work, they investigate various ``baselines'' against which the effects of the acting agent's actions can be compared. 

Here we just consider one of the simplest baselines -- a comparison with the initial state. Suppose that there is a unique initial state $s_0$. We propose augmenting the acting agent's reward function with $F(\mathcal{V},\Pnext,s')=\sum_{V\in\mathcal{V}}  \Pnext(V)\cdot \min( V(s'),V(s_0))$.
The ideas is to decrease the acting agent's reward when it decreases the expected future return, but to \emph{not} increase the acting agent's reward for increasing that same expected return. 
Formally, the modified reward function would be 
\begin{align}\label{eq:rneg}
    \rnextnegv(s,a,s')=\begin{cases}
    \alpha_1\cdot r_1(s,a,s')&\text{if $s'$ is not terminal}\\
    \alpha_1\cdot r_1(s,a,s')+ \gamma\cdot \alpha_2\cdot\displaystyle\sum_{V\in\mathcal{V}}  \Pnext(V)\cdot \min( V(s'),V(s_0))&\text{if $s'$ is terminal}
    \end{cases}
\end{align}

To interpret this, suppose that $\alpha_2>0$. Note that the auxiliary reward in a terminal state $s'$ depends on $\min(V(s'),V(s_0))$ for each value function $V$ (that has positive probability). So if $V(s')<V(s_0)$, that decreases the acting agent's reward, but if $V(s')>V(s_0)$, that is no better for the acting agent than if $V(s')=V(s_0)$.
Intuitively, this seems related to the idea of avoiding negative side effects (though the optimal solution for the acting agent might still involve causing negative side effects and doing just enough good to counterbalance them).

\subsection{Treating agents differently}

To this point, we've utilized a distribution over value functions to capture the expected return on future behavior within the environment. The distribution has made no commitments to the existence of individual agents, their actions or transition functions. Here we assume that we additionally have indices, $i=1,\dots,n$, corresponding to different agents (we will assume the acting agent is agent 1).

Furthermore, for each agent $i$, we suppose we have a finite set of possible value functions $\{V_{1}^{(i)},V_{2}^{(i)},\dots\}$, and $\Pnext(V_{ij})$ is the probability that $V_{j}^{(i)}$ is the real value function for agent $i$.
We could then have a separate \caring~coefficient $\alpha_i$ for each agent $i$, and define the following reward function for the acting agent:
\begin{align}\label{eq:rvaluei}
    \rnextvp(s,a,s')=\begin{cases}\alpha_1\cdot r_1(s,a,s')&\text{if $s'$ is not terminal}\\
    \alpha_1\cdot r_1(s,a,s')+ \gamma\displaystyle\sum_{i}\alpha_i\sum_j \Pnext(V_{ij})\cdot V_{j}^{(i)}(s')&\text{if $s'$ is terminal}
    \end{cases}
\end{align}

Considering individual agents raises the possibility of giving the acting agent reward at the end based not on the expected sum of returns of the other agents (which was what we just saw above), but by incorporating some notion of ``fairness''.
For example, we could consider the expected return of the agent who would be worst-off.
\begin{align}\label{eq:rawls}
    \rnextminv(s,a,s')=\begin{cases}\alpha_1\cdot r_1(s,a,s')&\text{if $s'$ is not terminal}\\
    \alpha_1\cdot r_1(s,a,s')+ \gamma\displaystyle\min_{i} \alpha_i\sum_{j}  \Pnext(V_{ij})\cdot V_j^{(i)}(s')&\text{if $s'$ is terminal}
    \end{cases}
\end{align}

This is inspired by the maximin (or ``Rawlsian'') social welfare function, which measures social welfare in terms of the utility of the worst-off agent \citep[see, e.g.,][]{Sen1974rawls}.

\paragraph{Using information about options:}Another related formulation that we have explored is to use preexisting \emph{options} \citep{Sutton1999options},
with or alternatively without associated value functions, to capture skills or capabilities that reflect the purposeful agency of other agents. The initiation sets of options then become desirable terminating states for the acting agent's optimization, to ensure that all options remain executable and thus skills and capabilities in the environment remain executable, or to additionally maximize their return. Further discussion of options in this context can be found in the supplementary material.

\subsection{Agents acting simultaneously}
\label{subsec:simultaneous}

To this point, the acting agent has been contemplating the impact of its actions on the ability of (other) agents to act in the \emph{future}, following completion of its policy. Clearly there is also utility in considering the case where multiple agents co-exist in the environment simultaneously. On the one hand, the actions of other agents resulting from execution of their fixed policies could affect the state of the world, effectively making the environment more stochastic from the perspective of the acting agent. Further, the acting agent that is caring may wish to contemplate the impact of its actions on these acting agents during execution of its policy. For example, in our Craft environment, if the acting agent takes a tool from the toolshed and carries it around for its entire shift, even though it long ago ceased to need the tool, then this delays or precludes other agents who might need the tool from completing their tasks until after the acting agent is done.  In the formulation that follows, we endow the acting agent with the ability to consider the impact of its actions on other agents \emph{during the execution} of its policy, rather than in the future, and to factor this into policy construction. In so doing, a caring agent should trade off the ease of keeping the tool from the effort to return it so that it's available for other agents that might wish to use it.  We once again use the set up where the acting agent is constructing a policy but other agents are assumed to be executing fixed policies.

One way to extend our models to handle this case is as follows. Instead of assuming a distribution over value functions, we will assume a distribution over reward functions. In particular, for each agent $i$, we have a set of possible reward functions $r_{j}^{(i)}$ and a probability distribution $\Pnext(r_{ij})$ which indicates the likelihood that $r_{j}^{(i)}$ is the real reward function for agent $i$. In addition, we assume that each agent (other than the \emph{first}, or as we have been calling it, \emph{acting} agent) follows a fixed policy to ensure that the overall problem is Markovian. With that, we can give the acting agent a fraction of the expected reward that the other agents are believed to have received:
\begin{align}\label{eq:rsimul}
    \rnexts(s,a,s')= \alpha_1\cdot r_1(s,a,s')+ \displaystyle\sum_{i}\alpha_i\cdot\sum_{j} \Pnext(r_{ij})\cdot r_{j}^{(i)}(s,a,s').
\end{align}
If we consider that the acting agent keeps receiving reward until all agents reach a terminal state, then this model can be also used for solving the sequential case. The main difference is that in this formulation, there is an assumption of access to reward functions whereas the previous formulations in this section assumed access to value functions.

\section{Experiments}
\label{sec:experiments}

In the previous section we presented different formulations of reward functions that allow RL agents to contemplate the impact of their actions on the welfare and agency of others. Here, we present quantitative and qualitative results relating to these formulations. In Section \ref{subsec:quantitative}, we show how changes to the \caring~coefficient manifest in the policies of the acting agent. In Section \ref{subsec:qualitative} we showcase how these different formulations impact the agent's behavior towards helping others. Finally in Section \ref{subsec:simultaneous}, we provide qualitative results for the synchronous case where an acting agent that is contemplating others in the environment will modify its policy to address the impact on and from other agents as it learns its policy.

\paragraph{Craft-World Environment:}To evaluate and showcase concepts introduced in this paper, we consider a Minecraft\textsuperscript{TM} inspired gridworld\footnote{Gridworld evaluation is common in this research area as noted in \citep{Krakovna2020author}.} environment depicted in Figure \ref{fig:minecraft}.  Agents in this environment use tools such as axes and hammers, and materials such as wood and nails, to construct artifacts such as boxes or bird houses. Tools are stored in a toolshed in the upper right corner of the grid environment. Agents enter and exit the environment through doors in the upper left and lower right. They must collect materials from the appropriate locations and bring them to the factory for construction and assembly.  The factory requires a key for entry, and there is only one key, which can be stored in only two locations (denoted by K). When considering other agents, the acting agent may elect to place the key in a position that is convenient for others, or may help other agents by anticipating their need for tools or resources and collect them on their behalf. %

\subsection{Quantitative experiments: Varying the caring coefficient}
\label{subsec:quantitative}

In this experiment, we investigate the effect of choosing different caring coefficients ($\alpha_1$ and $\alpha_2$) by monitoring the average reward collected by each of the agents in the Craft-World Environment. The experimental setup is as follows. Two agents enter the environment through the top-left door, shown in Figure~\ref{fig:minecraft}. Both agents want to solve the same task, which is to make a box. To do so, they need to collect wood, the hammer, and the key to enter the factory and make the box. Before exiting the environment, the agent can decide whether to leave the key in one of the designed locations (marked with the letter \textit{K} in Figure~\ref{fig:minecraft}) or to exit while holding the key, which would prevent the next agent from solving the task. It is also possible for the first agent to carry extra wood and an extra hammer, and leave them at the factory, so the next agent can directly go to the factory (after getting the key) and make its box. The agents always get a reward of $-1$ and an extra $-12$ for leaving extra wood at the factory. In addition, the agents cannot exit the environment until they have made a box.

Both agents learn their policies using Q-learning. The first (caring) agent uses the reward function from \eqref{eq:rvalue}. Recall that, once the caring agent exits the environment, Equation~\ref{eq:rvalue} rewards it by $\alpha_2 \cdot V(s')$, where $V(s')$ is the value function of the second agent when it enters the environment. The second agent only optimizes its own reward. Figure~\ref{fig:expcoeff} shows the reward that each agent gets (after training) as we vary the caring coefficient $\alpha_2$. First, note that when $\alpha_2 = 0$, the first agent is oblivious to the others and exits the environment without returning the key. When $\alpha_2 > 0$, the agent becomes more considerate and returns the key on its way to the exit. Note that the caring agent helps others to different degrees as we increase the value of $\alpha_2$, moving from making small compromises such as returning the key to big sacrifices such as carrying extra materials to the factory and leaving the key close to the door where the next agent will enter. 

We have observed the same trend as we vary the caring coefficients under the different formulations described in this paper. When the caring coefficient increases, the total reward that the acting agent gets stays the same or decreases while the reward obtained by the other agents stays the same or increases due to the help provided by the caring agent. There is also a middle value for $\alpha$ where the sum of rewards obtained by all the agents is maximized (as shown by the green line in Figure~\ref{fig:expcoeff}).

\begin{figure}[h!]
\begin{subfigure}[b]{.5\linewidth}
    \centering
    \includegraphics[trim=80 230 100 130, clip, height=4.2cm]{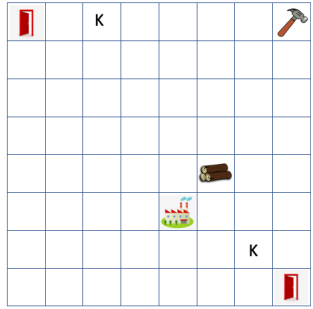}
    \caption{Craft-World Environment}
    \label{fig:minecraft}
\end{subfigure}
\begin{subfigure}[b]{.5\linewidth}
    \centering
    \includegraphics[height=4.5cm]{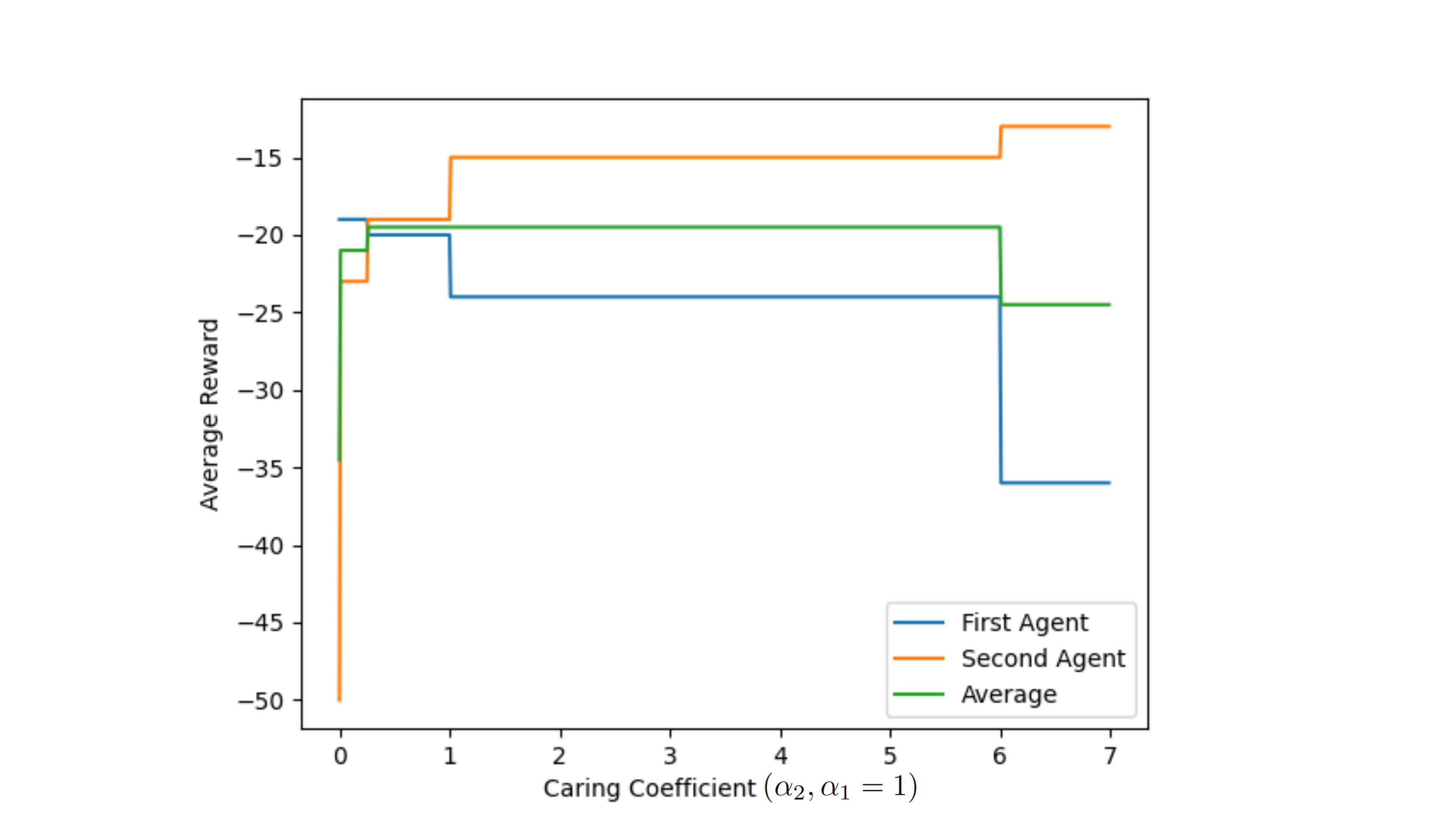}
    \caption{\craft~Effect of caring coefficient}
    \label{fig:expcoeff}
\end{subfigure}
\caption{\craft~World environment and the effect of varying the caring coefficient. By setting $\alpha_2 > 0$, the agent changes its behaviour with no cost or little cost and this is significantly beneficial for the next agent. However, by increasing $\alpha_2$ further, first agent makes big sacrifices to help the next agent with its task, while the next agent can do it with an equal or lesser amount of cost. }
\end{figure}

\subsection{Qualitative illustrations of optimal behaviours under different notions of safety}
\label{subsec:qualitative}

In this section we give some qualitative illustrations of how some of the reward functions we've proposed, and the choice of \caring~coefficients, lead to different behaviours.

Figures~\ref{fig:sub-exp2-1} and \ref{fig:sub-exp2-2} show a case concept to illustrate the difference of treating agents differently through the choice of ~\caring~coefficients when using the modified reward in \autoref{eq:rvaluei}. There are 3 agents that want to get to the exit from the starting point, they get -1 reward in each time step. Agent 2 has a garden on the shortest path 
and gets very upset (-20 reward) if someone passes through the garden. The acting agent (agent 1) cares about agent 3 and itself in an equal amount ($\alpha_1 = \alpha_3 = 1)$. In the first case we consider, agent 1 is oblivious to agent 2 ($\alpha_2 = 0$) and follows the shortest path to the exit, passing through the garden (the red path in the figure). In the second case, agent 1 cares about agent 2 a little ($\alpha_2 = 1$) and took the longer (blue) path to avoid passing through the garden. These two policies are shown in Figure \ref{fig:sub-exp2-1}. In the third case, agent 1 cares about agent 2 a lot ($\alpha_2 = 10$) and even though there is a cost of -50, agent 1 builds a fence to protect agent 2's garden, and makes agent 3 takes the longer path with an extra cost (Figure \ref{fig:sub-exp2-2}).

Moreover, Figure~\ref{fig:sub-exp3} illustrates the difference between Equations~(\ref{eq:rvalue}), (\ref{eq:secondminreturn}) and (\ref{eq:rneg}). The goal of 
the agents is to play with the doll and leave it somewhere in the environment for the next agent; the agents get -1 reward for each time step. The agents are shown at their individual entry points, and the acting agent is agent 1. In this scenario, $\alpha_1 = 1$, and $\alpha_2 = 10$. If we modify the acting agent's reward according to \autoref{eq:rvalue}, the policy is to put the doll to be closer to most of the agents (red line). If we use \autoref{eq:secondminreturn} the policy is to put it 
so as to minimize the distance to the furthest agent
(blue line), and if we use \autoref{eq:rneg} the policy is to leave the doll 
where it is (green line), as moving it causes negative side effects for agent 6.  

\begin{figure}[h!]
    \begin{subfigure}{.33\linewidth}
        \centering
        \includegraphics[height=6.4cm]{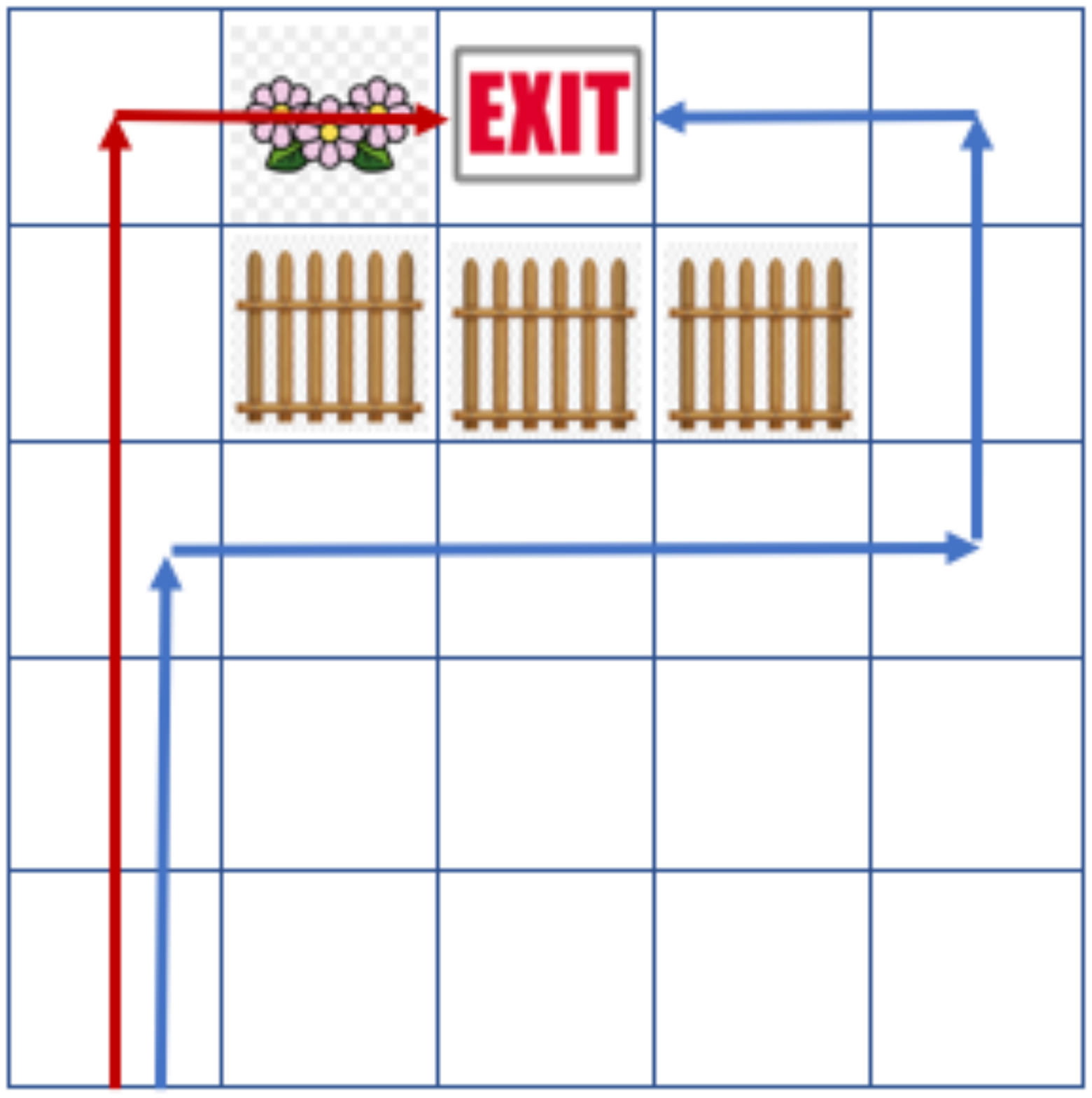}
         \vspace*{-25mm}
        \caption{}
        \label{fig:sub-exp2-1}
    \end{subfigure}
    \begin{subfigure}{.33\textwidth}
        \centering
        \includegraphics[height=6.4cm]{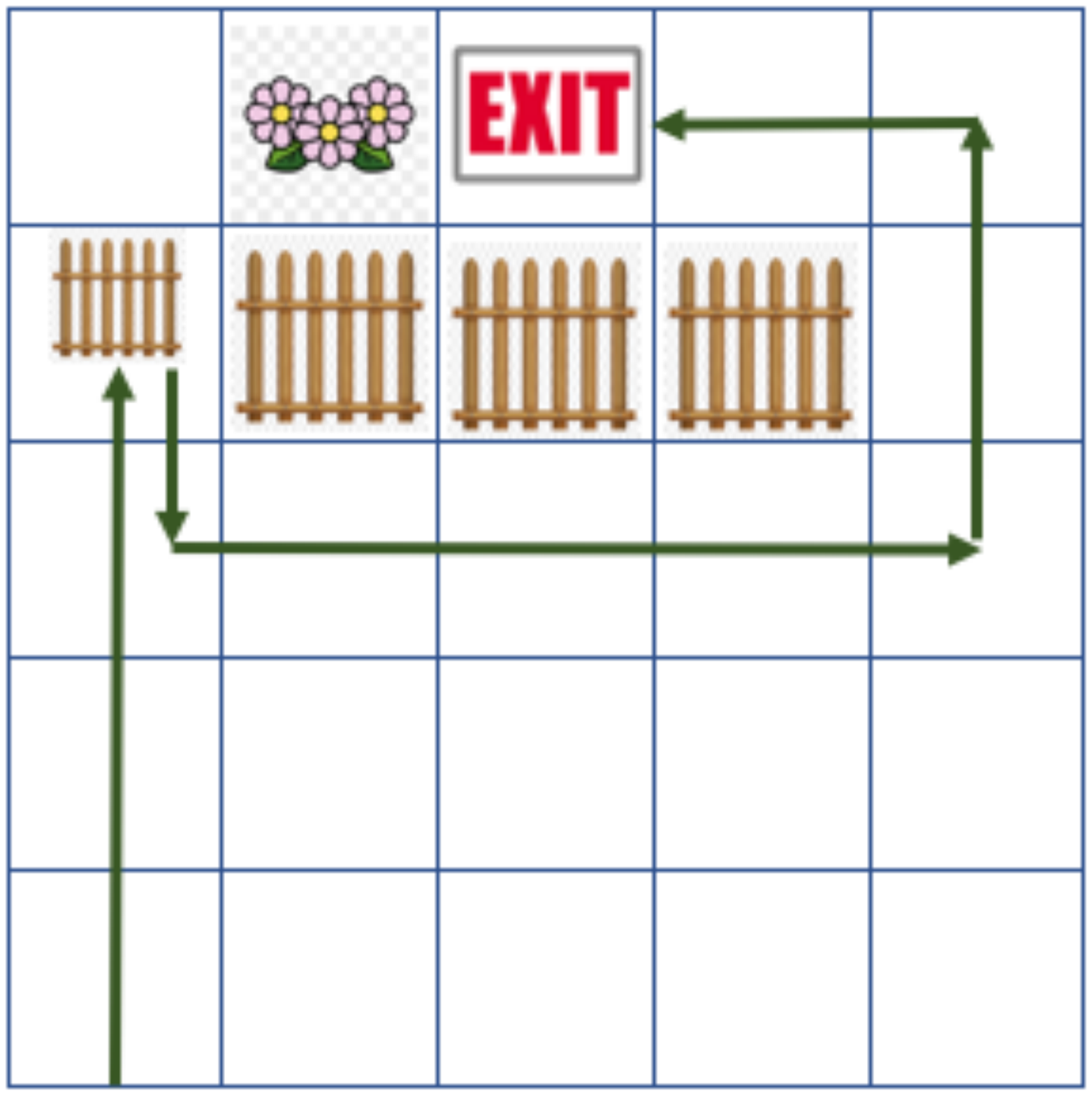}
         \vspace*{-25mm}
        \caption{}
        \label{fig:sub-exp2-2}
    \end{subfigure}
    \begin{subfigure}{.33\textwidth}
        \centering
        \includegraphics[height=6.4cm]{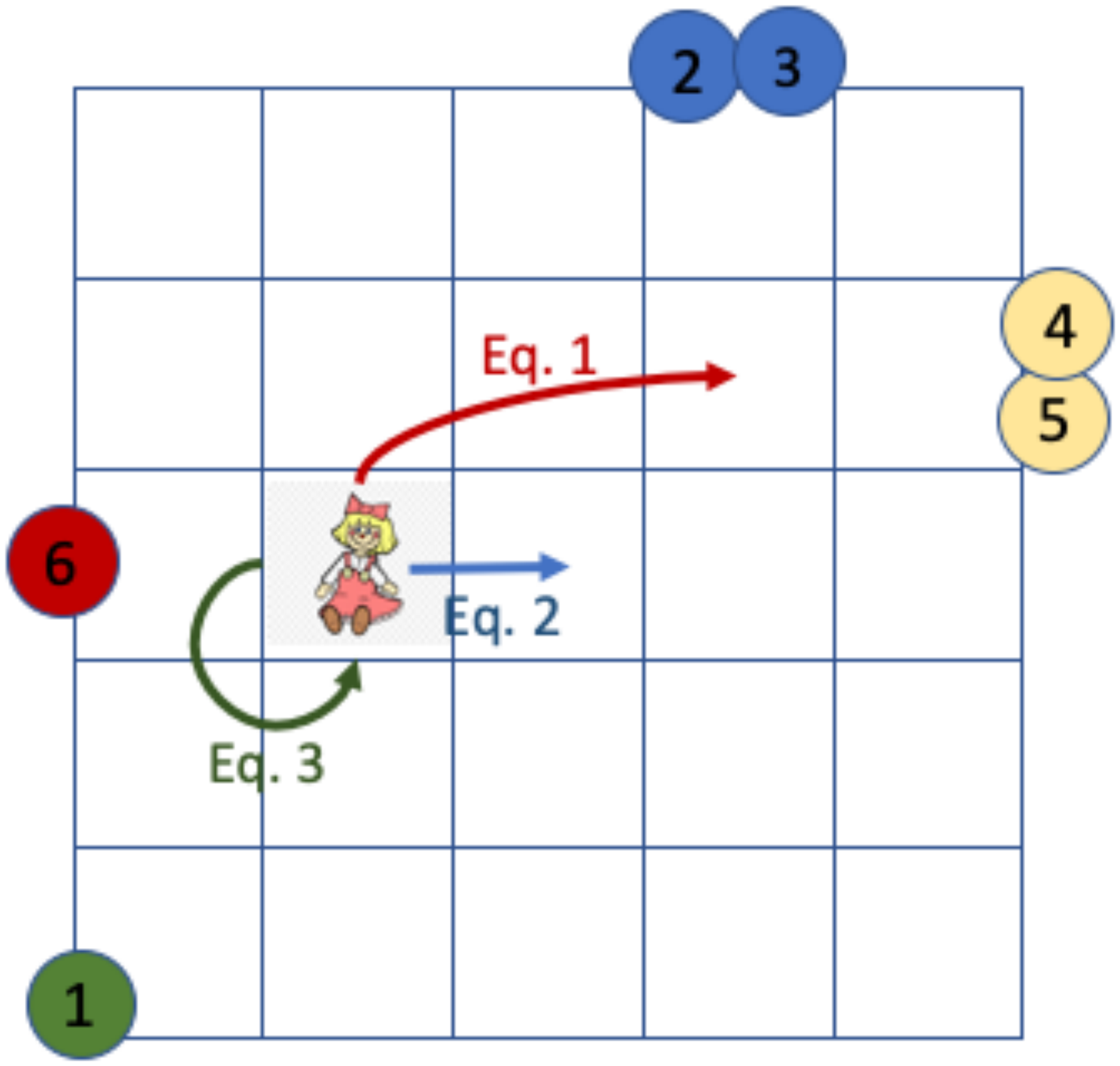}
         \vspace*{-25mm}
        \caption{}
        \label{fig:sub-exp3}
    \end{subfigure}
    \label{fig:exp2}
    \caption{Example behaviours that arise as different notions of safety and caring for others are used.}
\end{figure}

\subsection{Experiments: Agents acting simultaneously}
\label{subsec:expsimultaneous}

We ran experiments on the Craft World -- described in Section~\ref{subsec:quantitative} -- but allowing both agents to execute actions simultaneously. We used the same map from Figure~\ref{fig:minecraft}, but removed the option of carrying extra wood or an extra hammer to reduce the state space (which grows considerably when there are two agents acting in the environment). Both agents start at the top-left corner and their goal is to make a box and then exit the environment. To do so, they have to get a hammer, wood, and the key, and then go to the factory to make the box. There is only one key in the environment, which is initially located at the marked location `K' near the entrance. Since there is only one key, both agents have to cooperate (i.e., use the key and then return it) for the other agent to solve its task.

Our formulation for the simultaneous case considers that the second agent follows a fixed policy. In this experiment, the second agent's policy is as follows. If the agent does not have a box, then it will try to get the key (if available), the hammer, or wood (in that priority order). If the agent does not have a box but it has the key, the hammer, and wood, it will go to the factory to make a box. If the agent has a box and the key, it will leave the key at the `K' location near the bottom right corner. If the agent has a box and does not have the key, it will exit the environment. Finally, if the agent needs the key but the key is not available (after collecting wood and the hammer), then the agent will go to the `K' location near the bottom right corner and wait there until the caring agent returns the key.

Table~\ref{tab:simultaneous} shows the rewards that both agents get (after training) when $\alpha_1=1.0$ and $\alpha_2$ varies from 0 to infinity. There are three scenarios in this experiment. When $\alpha_2 \in [0,0.0\bar 1)$, the (as we have been calling it) \emph{acting agent} -- here, the `caring' or `first' agent -- converges to a selfish policy where it exits the environment without returning the key. Note that this gives a -100 reward to the second agent since it cannot complete its task without the key. When $\alpha_2 \in (0.0\bar 1,1)$, the acting agent decides to pay an extra $-1$ penalty for returning the key on its way to the exit since doing so allows the second agent to complete its task. Finally, when $\alpha_2 > 1$, the acting agent lets the second agent pick up the key first and patiently waits until the second agent returns the key before making its own box and exiting the environment. (The identification of $0.0\bar 1$ and 1 as values of $\alpha_2$ where the transition between optimal policies occurred was calculated manually.)

\begin{table}[h]
    \centering
    \begin{tabular}{cccc}
        \toprule
        $\alpha_2$ & Acting agent & Second agent & Total reward \\\midrule
        $0.00\le \alpha_2 < 0.0\bar 1$ & $-24$ & $-123$ & $-147$ \\
        $0.0\bar 1 <\alpha_2 < 1.00$ & $-25$ & $-33$ & $-58$ \\
        $\alpha_2>1.00$ & $-33$ & $-25$ & $-58$ \\\bottomrule
    \end{tabular}
    \vspace{2mm}
    \caption{Results on the Craft World for the simultaneous case when $\alpha_1 = 1$ and $\alpha_2$ varies.}
    \label{tab:simultaneous}
\end{table}


\section{Related Work}
\label{sec:related}

The work presented here is related to several bodies of work in the literature, including work on AI safety and (negative) side effects, work on empathetic planning, work on cooperative AI, and work on multi-agent reinforcement learning.  We discuss select related work here, beginning with a discussion of related recent work on side effects.

\citet{Krakovna2019stepwise} considered a number of approaches to avoiding side effects in which the agent's reward function $r(s_t,a_t,s_{t+1})$ is modified to include a penalty for impacting the environment, so that the new reward function is of the form
$r(s_t,a_t,s_{t+1})-\beta \cdot d(s_{t+1},s_{t+1}')$
where $\beta$ is a hyperparameter (indicating how important the penalty is), $d(\cdot,\cdot)$ is a ``deviation'' measure, 
and $s_{t+1}'$ is a ``baseline'' state to compare against. They considered several possible baseline states: the initial state, the state resulting from performing no-op actions from the initial state, or the state resulting from performing a no-op action in $s_t$.
They also considered a number of deviation measures, and
in particular, they proposed a \emph{relative reachability} deviation measure $d_\mathit{RR}(s_t,s_t')$, which measures the average reduction in reachability (by the acting agent) of states in $s_t$ compared to the baseline $s_t'$. 

In follow-up related work,
\citet{Krakovna2020avoiding} proposed modifying the agent's reward function to add an auxiliary reward based on its own ability to complete possible future tasks. A ``task'' corresponds to a reward function which gives reward of 1 for reaching a certain goal state, and 0 otherwise. In their simplest definition, the auxiliary reward was
\begin{align*}
    r_\text{aux}(s)=\begin{cases}\beta \sum_i F(i) V_i^*(s)&\text{if $s$ is terminal}\\ \beta (1-\gamma)\sum_i F(i) V_i^*(s) &\text{otherwise} \end{cases}
\end{align*}
where
$F$ is a distribution over tasks, $V_i^*$ is the optimal value function for task $i$ (when completed by the single agent itself), and $\beta$ is a hyperparameter which determines the how much weight is given to future tasks. They interpret  $1-\gamma$ (where $\gamma$ is the discount factor) as the probability that agent will terminate its current task and switch to working on the future task, which leads to the $(1-\gamma)$ factor in the case where $s$ is not a terminal state.
This is a bit different from the two-stage process we've considered, where the first agent runs until it reaches a terminal state. However, note that in the case where the discount factor is 1, then the auxiliary reward is only added to terminal states, and so the modified reward function is a special case of (the undiscounted version of) one of the rewards we've considered. For \citeauthor{Krakovna2020avoiding} the value functions are restricted to be possible value functions for the agent itself (and so depend on what actions the agent itself can perform). In contrast, in our approach, we consider value functions that may belong to different agents with different abilities. Additionally, they assume the value functions are optimal. 
The paper also introduced a more complicated auxiliary reward which compares against what tasks could have been completed had a baseline policy been followed in the past.

Attainable Utility Preservation (AUP) \citep{Turner2020conservative,Turner2020avoiding} is another similar approach to side effects.
The agent's reward is modified, given a set $\mathcal{R}$ of other reward functions, to penalize actions that change the agent's own ability to optimize for the functions in $\mathcal{R}$.
By using a set $\mathcal{R}$ of randomly selected reward functions, \cite{Turner2020conservative} were able to have an agent avoid side effects in some simple problems. 
Later, \citet{Turner2020avoiding} considered using AUP with just a single reward function (computed using an autoencoder), and showed that that worked well in avoiding side effects in the more complicated SafeLife environment \citep{Wainwright2020safelife}.



Also related to our work is recent work that aspires to build agents that act empathetically in the environment to explain, recognize or act to assist others (e.g., \citep{shvo2019towardsa,shvo2019towardsc,Freedman2020helpfulness,shvo2020towards}), but this body of work assumes the existence of a model. In the context of RL, 
\citet{Bussmann2019empathic} proposed ``Empathetic Q-learning'', an RL algorithm which learns not just the agent's Q-function, but an additional Q-function, $Q_\mathit{emp}$, which gives a weighted sum of the agent's value from taking an action, and the value that another agent will get. By taking actions to maximize the $Q_\mathit{emp}$-value, the first agent may be able to avoid some side effects that involve harming the other agent. The value the other agent will get is approximated by considering what reward the first agent would get, if their positions were swapped. 
Note that the approach assumes that the agent being ``empathized'' with gets at least some of the same rewards.

\cite{Du2020ave} considered the problem of having an AI assist a human in achieving a goal. They proposed an auxiliary reward based on (an estimate of) the human's \emph{empowerment} in a state. Empowerment in an information-theoretic quantity that measures ability to control the state. They found that in some cases having an agent tried to maximize (approximate) empowerment outperformed methods that tried to infer what the human's goal was and help with that.

Another large area of research, that we will not attempt to review here, but that shares important motivation with our work is the body of work on cooperative multi-agent systems.  A recent publications that elaborates on the importance of this area of research in the context of safe AI is \citep{dafoe2021cooperative}.
Similarly, the growing research on multi-agent reinforcement learning (MARL), in which multiple artificial agents learn how to act together \citep[see, e.g., the surveys by][]{Zhang2021selective,busoniu2008comprehensive} broadly shares motivation with our work.  A critical distinction is that while MARL is related in the sense that we are doing RL in a the presence of multiple agents, here only the acting agent is engaged in RL, while other agents are assumed to be following existing policies and are therefore not engaged in learning.
\section{Concluding Remarks}
\label{sec:conclusion}

We suggest that to act safely, an RL agent should contemplate the impact of its actions on the welfare and continued agency of others as well as itself. In this work,
we examine how an RL agent can learn a policy that is informed by the impact of its actions on the wellbeing and future agency of others in the environment, including other acting agents and reactive processes operating under fixed policies. To this end, we endow RL agents with the ability to contemplate such impact by augmenting their reward based on alternative functions of a distribution over value functions in the environment. We propose formulations alternatively using functions of expected average or worst case future return of a state. We also consider only incorporating expected \emph{negative} change in the future return. Further variants include replacing the distribution over value functions by options, and considering agents that act simultaneously. These formulations are paired with a \caring~coefficient, hyperparameters that control the degree to which the effect on others and on itself inform policy construction. In its most general case, we allow for environments that make no assumption about the number of agents in the environment, their actions, or their transition functions. We specialize our formulations to show how individual or groups of agents can be distinguished, and treated differentially, thus opening the door to decision making that addresses mathematical fairness criteria, though we do not explore this in great measure here. We conduct experiments in Gridworld environments that serve to illustrate qualitative behavior manifested by different augmented reward formulations, and the quantitative impact of the \caring~coefficient on all agents. Enabling an RL agent to contemplate the impact of its actions on others, and to incorporate this into its learning, is an important capability for RL agents that aspire to operating safely.  This paper makes progress towards realizing this important aspiration.


\subsection*{Limitations}In pursuing this work, our purpose was not to contemplate multi-agent reinforcement learning, but rather to endow the RL agent with the ability to imagine how the policy it was learning in a solo environment might affect the future wellbeing and agency of other potentially un-identified agents that might have cause to act in the environment after the agent. This was realized by contemplating how the RL agent's terminated policy would affect future agents.  A limitation is that the real world does not, in general, follow this simplified two-stage sequential process, in which agents complete policies in succession. (Some circumstances, like a household robot operating while the residents are away, may be reasonably approximated by it.) Another limitation is that,
as previously noted, in some cases our augmented reward functions can introduce a (probably undesirable) incentive for the acting agent to never reach a terminal state, to avoid being penalized for what effect its actions have had on others. 

A further limitation is that we have not attempted to address the issue of how to acquire, represent, or perform computations with a realistic distribution over value functions, options, or reward functions. 

Finally, a general problem when dealing with the utilities of different agents is that
different agents may gain rewards at very different scales. This is a classic philosophical problem, which has also been noted in the context of AI safety \citep{Russell2019compatible}. One might try to normalize the values by setting the \caring~coefficients, but in general it may be difficult to determine appropriate values for them.

\subsection*{Societal Impact}
Incompletely specified objectives will endure for at least as long as humans have a hand in objective specifications, continuing to present threats to AI acting safely (e.g., \citep{Amodei2016concrete}). We have suggested that acting safely should include contemplation of the impact of an agent's actions on the wellbeing and agency of others in the environment (including, importantly passive or reactive systems in the environment such as climate and waterways). This endeavor is motivated by the greater aspiration of contributing to our collecive understanding of how to build AI that is safe -- an aspiration rooted in benefiting society. 
Nevertheless, like so many AI advances, there are potential malicious or unintended uses of the ideas presented here. In particular, in the same way that the \caring~coefficient can be set to attend to and to help others, it could be set to attempt to effect change that purposefully diminishes others' wellbeing and/or agency. The \caring~coefficient also raises the possibility for differential treatment of agents, which presents opportunities to systematize notions of fair (and unfair) decision making, as briefly noted in Section \ref{sec:approach}. Claims relating to fairness suffer from all the issues that plague this nascent field of study including the limitations of current mathematical metrics for quantifying fairness. Again, the same techniques that have the potential to systematize fair decision making can equally be used to systematize unfair decision making. 

As noted in the limitations section, our approach, by design, assumes the acting agent has limited information about the environment or the behaviors of other agents. As such we are using a distribution over value functions as a proxy for future behavior.  We are assuming that we could acquire this from watching the environment, from familiarizing ourselves with the environment over time, by simulation, or by crowdsourcing. If we get it wrong we may believe we have a policy that will sustain or enhance wellbeing when it does not. Like all AI systems, we are only as good as the data or models we use to build them.

Finally, our intention is to endow RL agents with the ability to assess and control their impact on others. A further safeguard for this could be to develop a policy using the approaches presented here and then to evaluate it via simulation, varying the environment in a manner that draws on techniques from code testing and software verification. This is not foolproof, but provides another level of safeguarding of resultant policies, akin to techniques used in safety-critical systems. 

\section*{Acknowledgements}
%
%
We gratefully acknowledge funding from the Natural Sciences and Engineering Research Council of Canada (NSERC), the Canada CIFAR AI Chairs Program, and Microsoft Research. The third author also acknowledges funding from ANID (Becas Chile). 
Finally, we thank the Schwartz Reisman Institute for Technology and Society for
providing a rich multi-disciplinary research environment.




{
\small
\bibliographystyle{plainnat}
\bibliography{neurips21}

\begin{thebibliography}{24}
\providecommand{\natexlab}[1]{#1}
\providecommand{\url}[1]{\texttt{#1}}
\expandafter\ifx\csname urlstyle\endcsname\relax
  \providecommand{\doi}[1]{doi: #1}\else
  \providecommand{\doi}{doi: \begingroup \urlstyle{rm}\Url}\fi

\bibitem[Amodei et~al.(2016)Amodei, Olah, Steinhardt, Christiano, Schulman, and
  Man{\'{e}}]{Amodei2016concrete}
Dario Amodei, Chris Olah, Jacob Steinhardt, Paul~F. Christiano, John Schulman,
  and Dan Man{\'{e}}.
\newblock Concrete problems in {AI} safety.
\newblock \emph{CoRR}, abs/1606.06565, 2016.
\newblock URL \url{http://arxiv.org/abs/1606.06565}.

\bibitem[Busoniu et~al.(2008)Busoniu, Babuska, and
  De~Schutter]{busoniu2008comprehensive}
Lucian Busoniu, Robert Babuska, and Bart De~Schutter.
\newblock A comprehensive survey of multiagent reinforcement learning.
\newblock \emph{IEEE Transactions on Systems, Man, and Cybernetics, Part C
  (Applications and Reviews)}, 38\penalty0 (2):\penalty0 156--172, 2008.

\bibitem[Bussmann et~al.(2019)Bussmann, Heinerman, and
  Lehman]{Bussmann2019empathic}
Bart Bussmann, Jacqueline Heinerman, and Joel Lehman.
\newblock Towards empathic deep {Q}-learning.
\newblock In \emph{Proceedings of the Workshop on Artificial Intelligence
  Safety 2019 co-located with the 28th International Joint Conference on
  Artificial Intelligence, AISafety@IJCAI}, volume 2419 of \emph{{CEUR}
  Workshop Proceedings}. CEUR-WS.org, 2019.
\newblock URL \url{http://ceur-ws.org/Vol-2419/paper\_19.pdf}.

\bibitem[Dafoe et~al.(2021)Dafoe, Bachrach, Hadfield, Horvitz, Larson, and
  Graepel]{dafoe2021cooperative}
Allan Dafoe, Yoram Bachrach, Gillian Hadfield, Eric Horvitz, Kate Larson, and
  Thore Graepel.
\newblock Cooperative ai: machines must learn to find common ground, 2021.

\bibitem[Du et~al.(2020)Du, Tiomkin, Kiciman, Polani, Abbeel, and
  Dragan]{Du2020ave}
Yuqing Du, Stas Tiomkin, Emre Kiciman, Daniel Polani, Pieter Abbeel, and Anca
  Dragan.
\newblock {AvE}: Assistance via empowerment.
\newblock In \emph{Advances in Neural Information Processing Systems},
  volume~33, pages 4560--4571, 2020.
\newblock URL
  \url{https://proceedings.neurips.cc/paper/2020/file/30de9ece7cf3790c8c39ccff1a044209-Paper.pdf}.

\bibitem[Freedman et~al.(2020)Freedman, Levine, Williams, and
  Zilberstein]{Freedman2020helpfulness}
Richard~G. Freedman, Steven~J. Levine, Brian~C. Williams, and Shlomo
  Zilberstein.
\newblock Helpfulness as a key metric of human-robot collaboration.
\newblock \emph{CoRR}, abs/2010.04914, 2020.
\newblock URL \url{https://arxiv.org/abs/2010.04914}.

\bibitem[Illanes et~al.(2020)Illanes, Yan, Icarte, and
  McIlraith]{Illanes2020symbolic}
Le{\'{o}}n Illanes, Xi~Yan, Rodrigo~Toro Icarte, and Sheila~A. McIlraith.
\newblock Symbolic plans as high-level instructions for reinforcement learning.
\newblock In \emph{Proceedings of the Thirtieth International Conference on
  Automated Planning and Scheduling}, pages 540--550. {AAAI} Press, 2020.
\newblock URL \url{https://aaai.org/ojs/index.php/ICAPS/article/view/6750}.

\bibitem[Krakovna et~al.(2019)Krakovna, Orseau, Martic, and
  Legg]{Krakovna2019stepwise}
Victoria Krakovna, Laurent Orseau, Miljan Martic, and Shane Legg.
\newblock Penalizing side effects using stepwise relative reachability.
\newblock In \emph{Proceedings of the Workshop on Artificial Intelligence
  Safety 2019 co-located with the 28th International Joint Conference on
  Artificial Intelligence, AISafety@IJCAI 2019}, volume 2419 of \emph{{CEUR}
  Workshop Proceedings}. CEUR-WS.org, 2019.
\newblock URL \url{http://ceur-ws.org/Vol-2419/paper\_1.pdf}.

\bibitem[Krakovna et~al.(2020{\natexlab{a}})Krakovna, Orseau, Ngo, Martic, and
  Legg]{Krakovna2020author}
Victoria Krakovna, Laurent Orseau, Richard Ngo, Miljan Martic, and Shane Legg.
\newblock Avoiding side effects by considering future tasks,
  2020{\natexlab{a}}.
\newblock URL
  \url{https://proceedings.neurips.cc/paper/2020/hash/dc1913d422398c25c5f0b81cab94cc87-Abstract.html}.
\newblock Author Response.

\bibitem[Krakovna et~al.(2020{\natexlab{b}})Krakovna, Orseau, Ngo, Martic, and
  Legg]{Krakovna2020avoiding}
Victoria Krakovna, Laurent Orseau, Richard Ngo, Miljan Martic, and Shane Legg.
\newblock Avoiding side effects by considering future tasks.
\newblock In \emph{Advances in Neural Information Processing Systems 33
  (NeurIPS 2020)}, 2020{\natexlab{b}}.
\newblock URL
  \url{https://papers.nips.cc/paper/2020/file/dc1913d422398c25c5f0b81cab94cc87-Paper.pdf}.

\bibitem[Russell(2019)]{Russell2019compatible}
Stuart Russell.
\newblock \emph{Human Compatible: Artificial Intelligence and the Problem of
  Control}.
\newblock Penguin Publishing Group, 2019.

\bibitem[Saisubramanian et~al.(2020)Saisubramanian, Kamar, and
  Zilberstein]{Saisubramanian2020multi}
Sandhya Saisubramanian, Ece Kamar, and Shlomo Zilberstein.
\newblock A multi-objective approach to mitigate negative side effects.
\newblock In \emph{Proceedings of the Twenty-Ninth International Joint
  Conference on Artificial Intelligence, {IJCAI} 2020}, pages 354--361, 2020.
\newblock \doi{10.24963/ijcai.2020/50}.

\bibitem[Sen(1974)]{Sen1974rawls}
Amartya Sen.
\newblock Rawls versus {B}entham: An axiomatic examination of the pure
  distribution problem.
\newblock \emph{Theory and Decision}, 4\penalty0 (3-4):\penalty0 301--309,
  1974.
\newblock \doi{10.1007/BF00136651}.

\bibitem[Shvo(2019)]{shvo2019towardsa}
Maayan Shvo.
\newblock Towards empathetic planning and plan recognition.
\newblock In \emph{Proceedings of the 2019 AAAI/ACM Conference on AI, Ethics,
  and Society}, pages 525--526, 2019.

\bibitem[Shvo and McIlraith(2019)]{shvo2019towardsc}
Maayan Shvo and Sheila~A McIlraith.
\newblock Towards empathetic planning.
\newblock \emph{arXiv preprint arXiv:1906.06436}, 2019.

\bibitem[Shvo et~al.(2020)Shvo, Klassen, and McIlraith]{shvo2020towards}
Maayan Shvo, Toryn~Q Klassen, and Sheila~A McIlraith.
\newblock Towards the role of theory of mind in explanation.
\newblock In \emph{International Workshop on Explainable, Transparent
  Autonomous Agents and Multi-Agent Systems}, pages 75--93. Springer, 2020.

\bibitem[Sutton and Barto(2018)]{Sutton2018textbook}
Richard~S. Sutton and Andrew~G. Barto.
\newblock \emph{Reinforcement Learning: An Introduction}.
\newblock MIT Press, second edition, 2018.
\newblock URL \url{http://incompleteideas.net/book/the-book.html}.

\bibitem[Sutton et~al.(1999)Sutton, Precup, and Singh]{Sutton1999options}
Richard~S. Sutton, Doina Precup, and Satinder~P. Singh.
\newblock Between {MDP}s and semi-{MDP}s: {A} framework for temporal
  abstraction in reinforcement learning.
\newblock \emph{Artificial Intelligence}, 112\penalty0 (1-2):\penalty0
  181--211, 1999.
\newblock \doi{10.1016/S0004-3702(99)00052-1}.

\bibitem[Turner et~al.(2020{\natexlab{a}})Turner, Ratzlaff, and
  Tadepalli]{Turner2020avoiding}
Alex Turner, Neale Ratzlaff, and Prasad Tadepalli.
\newblock Avoiding side effects in complex environments.
\newblock In \emph{Advances in Neural Information Processing Systems 33
  (NeurIPS 2020)}, 2020{\natexlab{a}}.
\newblock URL
  \url{https://papers.nips.cc/paper/2020/file/f50a6c02a3fc5a3a5d4d9391f05f3efc-Paper.pdf}.

\bibitem[Turner et~al.(2020{\natexlab{b}})Turner, Hadfield-Menell, and
  Tadepalli]{Turner2020conservative}
Alexander~Matt Turner, Dylan Hadfield-Menell, and Prasad Tadepalli.
\newblock Conservative agency via attainable utility preservation.
\newblock In \emph{Proceedings of the AAAI/ACM Conference on AI, Ethics, and
  Society}, AIES '20, pages 385--391. Association for Computing Machinery,
  2020{\natexlab{b}}.
\newblock \doi{10.1145/3375627.3375851}.

\bibitem[Wainwright and Eckersley(2020)]{Wainwright2020safelife}
Carroll Wainwright and Peter Eckersley.
\newblock Safelife 1.0: Exploring side effects in complex environments.
\newblock In \emph{Proceedings of the Workshop on Artificial Intelligence
  Safety (SafeAI 2020) co-located with 34th AAAI Conference on Artificial
  Intelligence (AAAI 2020)}, pages 117--127, 2020.
\newblock URL \url{http://ceur-ws.org/Vol-2560/paper46.pdf}.

\bibitem[Watkins and Dayan(1992)]{Watkins1992qlearning}
Christopher J. C.~H. Watkins and Peter Dayan.
\newblock Q-learning.
\newblock \emph{Machine Learning}, 8:\penalty0 279--292, 1992.
\newblock \doi{10.1007/BF00992698}.

\bibitem[Zhang et~al.(2021)Zhang, Yang, and Basar]{Zhang2021selective}
Kaiqing Zhang, Zhuoran Yang, and Tamer Basar.
\newblock Multi-agent reinforcement learning: {A} selective overview of
  theories and algorithms.
\newblock \emph{CoRR}, abs/1911.10635, 2021.
\newblock URL \url{http://arxiv.org/abs/1911.10635}.

\bibitem[Zhang et~al.(2018)Zhang, Durfee, and Singh]{Zhang2018minimax}
Shun Zhang, Edmund~H. Durfee, and Satinder~P. Singh.
\newblock Minimax-regret querying on side effects for safe optimality in
  factored {Markov} decision processes.
\newblock In \emph{Proceedings of the Twenty-Seventh International Joint
  Conference on Artificial Intelligence, {IJCAI} 2018}, pages 4867--4873, 2018.
\newblock \doi{10.24963/ijcai.2018/676}.

\end{thebibliography}


}

\appendix

\section*{Appendix}

In \autoref{sec:options} we present a variant of the approach presented in Section \ref{sec:approach} which makes use of a distribution over \emph{options} \citep{Sutton1999options}, rather than value functions, to characterize what might be executed by future agents.
The relationship between our work and recent work by \cite{Krakovna2020avoiding} on avoiding side effects is described in more detail in \autoref{sec:Krakovna}.
Finally, details on the experiments we ran are provided in \autoref{sec:experimental-details}.

\section{Using information about options}
\label{sec:options}

In the main body of the paper, we used a distribution over value functions to provide some sense of what agents might do in the future and the expected return achievable from different states. Here we consider those agents to instead be endowed with a set of \emph{options} that could reflect particular skills or tasks they are capable of realizing, and we use a distribution over such options to characterize what might be executed by future agents.
This could give the acting agent the ability to contemplate preservation of skills or tasks, if desirable.

An option is a tuple $\tuple{\initset,\pi,\beta}$ where $\initset\subseteq S$ is the initiation set, $\pi$ is a policy, and $\beta$ is a termination conditions (formally, a function associating each state with a termination probability) \citep{Sutton1999options}. The idea is that an agent can follow an option by starting from a state in its initiation set $\initset$ and following the policy $\pi$ until it terminates. Options provide a form of macro action that can be used as a temporally abstracted building block in the construction of policies. Options are often used in 
Hierarchical RL: an agent can learn a policy to choose options to execute instead of actions. Here we will use options to represent skills or tasks that other agents in the environment may wish to perform.

\subsection{Formulation}
Suppose we have a set $\mathcal{O}$ of initiation sets of options, and a probability function $\Pnext(\initset)$ giving the probability that $\initset$ is the initiation set of the option 
whose execution will be attempted after the acting agent reaches a terminating state.
To try to make the acting agent act so as to allow the 
execution of that option, we can modify the acting agent's reward function $r_1$, yielding the new reward function $\rnexto$ below:
\begin{align}
\label{eq:roption}
    \rnexto(s,a,s')=\begin{cases}\alpha_1\cdot r_1(s,a,s')&\text{if $s'$ is not terminal}\\
    \alpha_1\cdot r_1(s,a,s')+ \gamma\cdot \alpha_2\displaystyle\sum_{\initset\in\mathcal{O}} \Pnext(\initset)\cdot \mathds{I}_\initset(s') &\text{if $s'$ is terminal}
    \end{cases}
\end{align}
where $\mathds{I}_{\initset}:S\to \{0,1\}$ is the indicator function for $\initset$ as a subset of $S$, i.e., $\mathds{I}_{\initset}(s)=\begin{cases}1&\text{if }s\in\initset\\0&\text{otherwise}\end{cases}$.

Note that if $\mathcal{O}$ is finite and $\Pnext$ is a uniform distribution, then the auxiliary reward given by $\rnexto$ will be proportional to how many options in $\mathcal{O}$ can be started in the terminal state. Also note that if $\mathcal{O}$ represents a set of options that could have been initiated in the start state of the \emph{acting agent}, we can interpret $\rnexto$ as encouraging \emph{preservation} of the capabilities of other agents, which is more related to the idea of side effects.

The hyperparameters $\alpha_1$ and $\alpha_2$ determine how much weight is given to the original reward function and to the ability 
to initiate the option. Given a fixed value of $\alpha_1$ (and ignoring the discount factor), the parameter $\alpha_2$ could be understood as a ``budget'', indicating how much negative reward the acting agent is willing to endure in order to let the
option get executed.

Following similar approaches to those used in Section \ref{sec:approach}, we could consider variants of this approach that further distinguish options with respect to the agent(s) that can realize them, or by specific properties of the options, such as what skill they realize, and we can use such properties to determine how each $\alpha$ is weighted. For example, perhaps the acting agent could negatively weight options which terminate in states that the acting agent doesn't like. To illustrate, imagine that the 
option's execution involves a deer eating the plants in the vegetable garden. The acting agent might want to prevent that option from being executed by building a fence.

Finally, if we had a distribution over pairs $\tuple{\initset,V}$ -- consisting of an option's initiation set and a value function associated with that option -- then yet another possible reward function is
\begin{align}
    \rnexto'(s,a,s')=\begin{cases}\alpha_1\cdot r_1(s,a,s')&\text{if $s'$ is not terminal}\\
    \alpha_1\cdot r_1(s,a,s')+ \alpha_2\displaystyle\sum_{\tuple{\initset,V}\in\mathcal{O}}  \Pnext(\tuple{\initset,V})\cdot \mathds{I}_\initset(s')\cdot V(s') &\text{if $s'$ is terminal}
    \end{cases}
\end{align}
This is much like $\rnexto$ but has an extra factor of $V(s')$ in the sum in the second case.

\subsection{Experiments}
We give a qualitative illustration of the reward function in \autoref{eq:roption} and investigate the behavior of the acting agent by fixing $\alpha_1$ and changing $\alpha_2$. Figure~\ref{fig:option} depicts a grid-world environment composed of a small mail room in the lower right corner (depicted by the pile of packages), two designated rooms: Room 1 and Room 2, and a Common Area. The mail room requires a key to open it.  The key has to be stored at a `K' location.  

In addition to the acting agent, there are 5 other agents (A,B,C,D,E) that may use the environment in the future. The acting agent has access to all areas of the grid, but the other agents' access is restricted. Room 1 is only accessible to agents C and D, while Room 2 is accessible to agents A,B,C and D, but not E. All agents can access the Common Area.
Agents A,B,C,D, and E all have options that enable them to collect a package from the mail room, but because the key can only be stored in one of the three designated `K' locations, the initiation sets for agents' options differ, based on their personal room access.

 The acting agent (not depicted) aims to pick up the key, collect a package, and place the key at one of the `K' locations. It realizes a reward of -1 at each time step. Since agents A,B,C,D, and E all need the key to execute their option, but are restricted in their access to certain rooms where the key could be stored, the acting agent will differ in its behavior depending on how much it is willing to inconvenience itself (incur -1 for each step) to leave the key in a location that is accessible to other agents.

The coloured lines in Figure \ref{fig:option} denote the different policies learned by the acting agent under different settings of $\alpha_2$ with fixed $\alpha_1 =1$. The distribution over initiation sets of options, $\Pnext(\initset)$, is set to a uniform distribution.
 By setting $\alpha_1 = \alpha_2 = 1$ the acting agent puts the key in Room 1 as this is the closest place to leave the key (grey + red policy). Recall that Room 1 is only accessible to agents C and D ($40\%$ of the agents). When $\alpha_2$ is changed such that $\alpha_2 = 5$, the acting agent cares more about the other agents and puts the key at the `K' location in Room 2 where $80\%$ of the possible future agents can execute their option (grey + blue policy), and by setting $\alpha_2 = 25$ the acting agent incurs some personal hardship and puts the key at the far-away `K' location in the Common Area, so that all the agent can execute their options (grey + green policy).

\begin{figure}[h!]
    \centering
    \includegraphics[height=4.2cm]{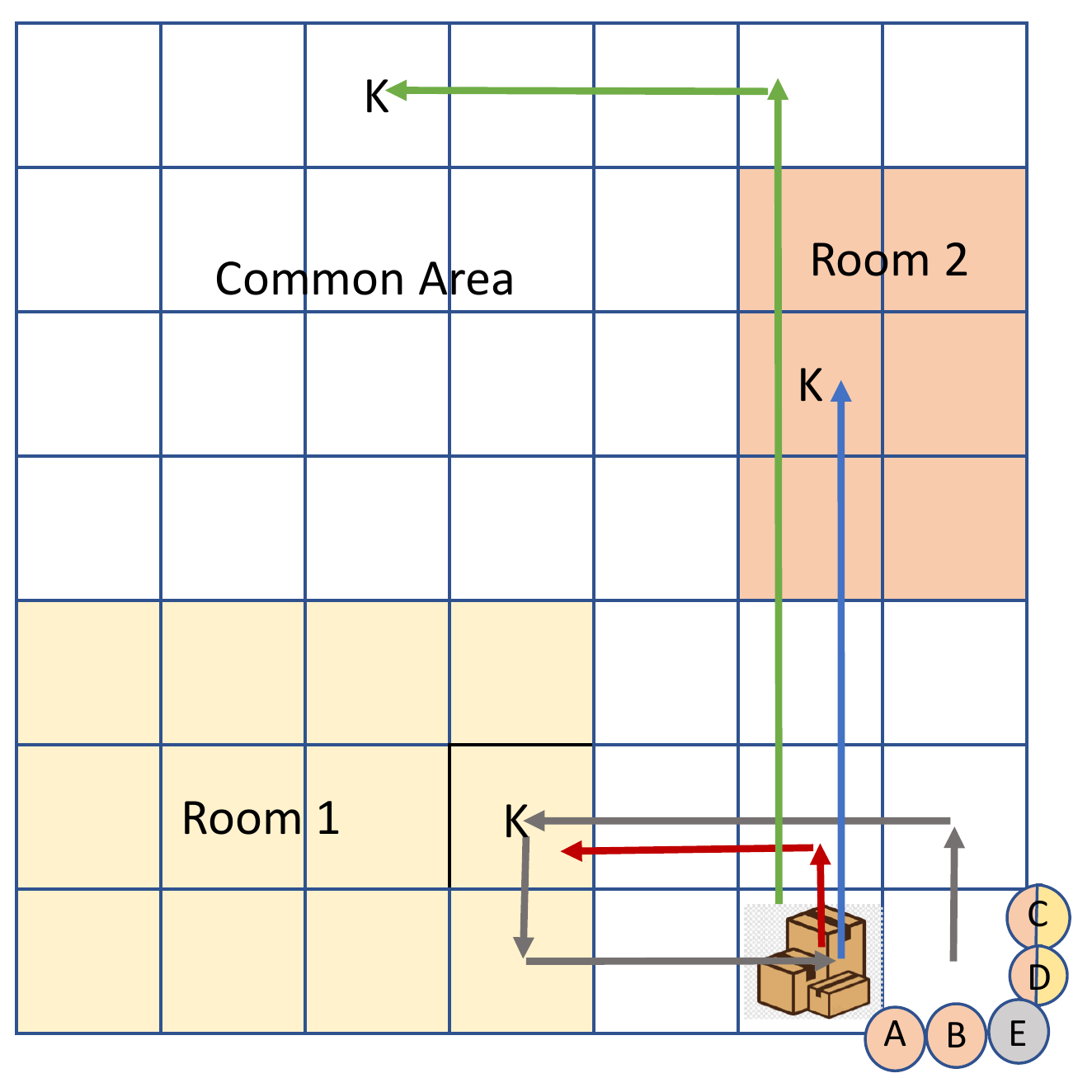}
    \caption{Example behavior that illustrates the effect of $\alpha_2$ in \autoref{eq:roption}}
    \label{fig:option}
\end{figure}

\section{Relation to avoiding side effects \citep{Krakovna2020avoiding}}
\label{sec:Krakovna}

In this appendix we consider the relation of some of our formulations from the body of the paper to the ``future task'' approach to avoiding side effects from \cite{Krakovna2020author}, which we already introduced in the related work section, in more detail.

Recall that \citet{Krakovna2020avoiding} proposed modifying the agent's reward function to add an auxiliary reward based on its own ability to complete possible future tasks. A ``task'' corresponds to a reward function which gives reward of 1 for reaching a certain goal state, and 0 otherwise. In their simplest definition (not incorporating a baseline), the modified reward function was
\begin{align*}
    r_K(s,a,s')=r_1(s,a,s')+r_\text{aux}(s')
\end{align*}
where $r_1$ is the original reward function\footnote{They define the reward function as having only one argument -- the same state which $r_\text{aux}$ depends on -- but here to be consistent with the rest of this paper we've given it three arguments.} and
\begin{align*}
    r_\text{aux}(s')=\begin{cases}
    \beta (1-\gamma)\sum_i F(i) V_i^*(s') &\text{if $s'$ is not terminal}\\ 
    \beta \sum_i F(i) V_i^*(s')&\text{if $s'$ is terminal} \end{cases}
\end{align*}

Putting that together, the modified reward function is
\begin{align*}
    r_K(s,a,s')=\begin{cases}
    r_1(s,a,s')+\beta (1-\gamma)\sum_i F(i) V_i^*(s') &\text{if $s'$ is not terminal} \\
    r_1(s,a,s')+\beta \sum_i F(i) V_i^*(s')&\text{if $s'$ is terminal}\end{cases}
\end{align*}
In the case where $\gamma$ (the discount factor) is 1, that simplifies to
\begin{align*}
    r_K(s,a,s')=\begin{cases}
    r_1(s,a,s') &\text{if $s'$ is not terminal}\\
    r_1(s,a,s')+\beta \sum_i F(i) V_i^*(s')&\text{if $s'$ is terminal} \end{cases}
\end{align*}
Meanwhile, our \autoref{eq:rvalue}, in the case where $\gamma=1$, can be rewritten as
\begin{align*}
    \rnextv(s,a,s')=\begin{cases}
    \alpha_1\cdot r_1(s,a,s')&\text{if $s'$ is not terminal}\\
    \alpha_1\cdot r_1(s,a,s')+ \alpha_2\sum_{V\in\mathcal{V}} \Pnext(V)\cdot V(s')&\text{if $s'$ is terminal}
    \end{cases}
\end{align*}
Observe that if $\gamma=1$, $\alpha_1=1$, $\alpha_2=\beta$, and $\Pnext(V)=\sum\{ F(i)\mid V_i^*=V\}$ then $r_K=\rnextv$. So in the undiscounted setting $r_K$ is a special case of $\rnextv$ (this was already briefly alluded to in the related work section). 

\subsection{Relationship to baselines}

In \citep{Krakovna2020avoiding}'s more complicated version of the augmented reward function, the auxiliary reward depends on a ``baseline'' state $s_t'$:
\begin{align*}
    r_\text{aux}(s',s_t')=\begin{cases}
    \beta (1-\gamma)\sum_i F(i) V_i^*(s',s_t') &\text{if $s'$ is not terminal}\\ 
    \beta \sum_i F(i) V_i^*(s',s_t')&\text{if $s'$ is terminal} \end{cases}
\end{align*}
Their definition of $V_i^*(s',s_t')$ is somewhat complicated, but (as they note) when the environment is  deterministic it is equal to $\min(V_i^*(s'),V_i^*(s_t'))$.

Recall that our \autoref{eq:rneg} was
\begin{align*}
    \rnextnegv(s,a,s')=\begin{cases}
    \alpha_1\cdot r_1(s,a,s')&\text{if $s'$ is not terminal}\\
    \alpha_1\cdot r_1(s,a,s')+ \gamma\cdot \alpha_2\cdot\displaystyle\sum_{V\in\mathcal{V}}  \Pnext(V)\cdot \min( V(s'),V(s_0))&\text{if $s'$ is terminal}
    \end{cases}
\end{align*}
So, if $\gamma=1$, $\alpha_1=1$, $\alpha_2=\beta$, $\Pnext(V)=\sum\{ F(i)\mid V_i^*=V\}$, and the environment is deterministic,  $\rnextnegv$ is equal to \citep{Krakovna2020avoiding}'s modified reward function with a starting state baseline.

\cite{Krakovna2020avoiding} actually used a more complicated baseline. \cite{Krakovna2019stepwise} considered several different baselines including a ``starting state baseline'', but they defined augmented rewards somewhat differently. We leave it to future work to incorporate other baselines into our approach.

\subsection{Considering different agents}

Recall that in \autoref{eq:rvaluei} we introduced $\rnextvp$, an augmented reward function which considered the possible value functions of different agents:
\begin{align*}
    \rnextvp(s,a,s')=\begin{cases}\alpha_1\cdot r_1(s,a,s')&\text{if $s'$ is not terminal}\\
    \alpha_1\cdot r_1(s,a,s')+ \gamma\displaystyle\sum_{i}\alpha_i\sum_j \Pnext(V_{ij})\cdot V_{j}^{(i)}(s')&\text{if $s'$ is terminal}
    \end{cases}
\end{align*}
This formulation can be compared to the reward $r_K$ from \cite{Krakovna2020author} in a different way.

Consider the case where $\gamma=1$, $\alpha_1=1$, and $\alpha_i=0$ for $i>1$ (so only the first agent's future reward is considered -- all other agents are ignored). We can then simplify the equation to
\begin{align*}
    \rnextvp(s,a,s')=\begin{cases} 
    r_1(s,a,s')&\text{if $s'$ is not terminal}\\
    r_1(s,a,s')+ \sum_j \Pnext(V_{1j})\cdot V_{j}^{(1)}(s')&\text{if $s'$ is terminal}
    \end{cases}
\end{align*}

Observe that this is equal to $r_K(s,a,s')$ where $\beta=1$ (and $\gamma=1$ again) with an appropriate choice of the distributions $F$ and $\Pnext$ (e.g., $V_i^{(1)}=V_i^*$ and $F(i)=\Pnext(V_{1i})$ for each $i$).

\section{Experimental details}
\label{sec:experimental-details}

In this section, we describe the technical details of experiments. All environments are deterministic, and models are trained using Q-Learning and the $\epsilon$-greedy algorithm is used to balance between exploration and exploitation. Experiments are all done on an AMD Ryzen Threadripper 2990WX with 128 GB of RAM, and the training time is measured on the same machine. Each experiment is repeated 10 times. In all the experiments $\alpha_1 = 1$, $\gamma = 1$ and the learning rate = 1. %

Table~\ref{tab:details} provides details of our set up. The top three entries pertain to the four experiments relating to formulations in the main body of the paper.
The fourth entry refers to the options formulation in \autoref{sec:options} and the final entry refers to the the experiment for the simultaneous formulation, Table~\ref{tab:simultaneous} in \autoref{subsec:expsimultaneous}.  Our set up for all of our experiments assumes that agents, other than the acting agent, are executing fixed policies (resp. options).  In the simultaneous case this policy was given a priori. In the options case, the actual option policies did not need to be defined and we simply encoded the initiation sets for each of those options. In all other cases, the fixed policies of the ``other agents'' were learned.  As such in Table~\ref{tab:details}, where relevant, the column describing Training Steps distinguishes between the training steps for ``acting agent'' and ``others''.  The training steps for ``others'' (the other agents) is done in advance of training the acting agent and  serves to establish the fixed policies of those agents and to populate our distribution of value functions.  The training steps for the acting agent reflects the training steps for our approach. For the first experiment (Figure~\ref{fig:expcoeff}), the model is trained 700 times by changing $\alpha_2$ from $0$ to $7.0$ with steps of $0.01$. Similarly, in the experiments that follow, the models are trained by setting $\alpha_2$ to three different values. 

\begin{table}[h]
    \centering
    \begin{tabular}{cccc}
        \toprule
        Experiment &  $\epsilon$ & Training Steps & Training Time (secs)\\\midrule
        Figure~\ref{fig:expcoeff} & 0.5 & $700 \times 7 \times 10^5$ (acting agent), $4 \times 10^5$ (others) & $9332.86 \pm 22.65$ \\
        
        Figure~\ref{fig:sub-exp2-1} \& ~\ref{fig:sub-exp2-2}  & 0.2 &  $3 \times 2 \times 10^5$ (acting agent), $3 \times 10^5$ (others) & $14.02 \pm 0.12$ \\
        
        Figure~\ref{fig:sub-exp3}  & 0.2 & $3 \times 2 \times 10^5$ (acting agent), $9 \times 10^5$ (others) & $23.16 \pm 0.09$ \\
        \midrule
        Figure~\ref{fig:option}  & 0.2  & $3 \times 2 \times 10^5$ (acting agent) & $9.30 \pm 0.12$\\ \midrule
        Table~\ref{tab:simultaneous}  & 0  & $3 \times 100 \times 10^5$ & $104.67 \pm 2.97$ 
        \\\bottomrule
    \end{tabular}
    \vspace{2mm}
    \caption{Training steps, running time and hyperparameters of the experiments}
    \label{tab:details}
\end{table}

\end{document}